\documentclass{sig-alternate}

\usepackage{etoolbox}
\makeatletter
\patchcmd{\maketitle}{\@copyrightspace}{}{}{}
\makeatother

\usepackage{amssymb,amsmath}
\usepackage{xspace}
\usepackage{tikz}
\usetikzlibrary{decorations.pathreplacing}
\usepackage{units}
\usepackage{graphicx}
\usepackage{algorithm,algorithmic}
\usepackage{pgfplots}
\usepackage{mathtools}

\newtheorem{theorem}{Theorem}

\newtheorem{lemma}[theorem]{Lemma}

\newtheorem{definition}{Definition}

\newcommand{\ea}{(1+1)~EA\xspace}

\newcommand{\N}{\ensuremath{{\mathbb N}}}

\newcommand{\R}{\ensuremath{{\mathbb R}}}
\newcommand{\pfix}{\ensuremath{p_\mathrm{fix}}}

\newcommand{\const}{11}
\newcommand{\Tpeak}{T_{\mathrm{peak}}}
\newcommand{\psuccess}{p_\mathrm{success}}
\newcommand{\loincrease}{d}

\newcommand{\s}{\ensuremath{\beta \Delta f}}
\newcommand{\df}{\ensuremath{\Delta f}}

\newcommand{\lo}[1]{\text{\sc LO}(#1)\xspace}

\newcommand{\onemax}{\text{\sc OneMax}\xspace}

\newcommand{\cliff}[1]{\text{\sc Cliff}_{#1}\xspace}

\newcommand{\balance}{\text{\sc Balance}\xspace}

\newcommand{\Onemax}{\onemax}

\newcommand{\ones}[1]{\ensuremath |#1|_1}
\newcommand{\zeros}[1]{\ensuremath |#1|_0}

\newcommand{\E}[1]{\text{E}\left(#1\right)}

\newcommand{\Prob}[1]{\mathrm{Pr}\left(#1\right)}
\newcommand{\poly}[1]{\text{poly}\left(#1\right)}

\DeclareMathOperator{\mut}{mut}

\usepackage{framed}
\colorlet{shadecolor}{gray!25}

\input{redefine-proof.tex}

\begin{document}

\title{First Steps Towards a Runtime Comparison of Natural and Artificial Evolution}

\numberofauthors{4}
\author{
	\alignauthor Tiago Paix\~{a}o\\
	\affaddr{IST Austria} \\
	\affaddr{Am Campus 1, 3400, Klosterneuburg} \\
	\alignauthor Jorge P\'{e}rez Heredia\\
	\affaddr{University of Sheffield} \\
	\affaddr{Sheffield, S1 4DP, United Kingdom} \\
\and
	\alignauthor Dirk Sudholt\\
	\affaddr{University of Sheffield} \\
	\affaddr{Sheffield, S1 4DP, United Kingdom} \\
	\alignauthor Barbora Trubenov\'{a}\\
	\affaddr{IST Austria} \\
	\affaddr{Am Campus 1, 3400, Klosterneuburg} \\
}

\maketitle

\begin{abstract}
	Evolutionary algorithms (EAs) form a popular optimisation paradigm inspired by natural evolution. In recent years the field of evolutionary computation has developed a rigorous analytical theory to analyse their runtime on many illustrative problems.
Here we apply this theory to a simple model of natural evolution. In the Strong Selection Weak Mutation (SSWM) evolutionary regime the time between occurrence of new mutations is much longer than the time it takes for a new beneficial mutation to take over the population. In this situation, the population only contains copies of one genotype and evolution can be modelled as a (1+1)-type process where the probability of accepting a new genotype (improvements or worsenings) depends on the change in fitness.

We present an initial runtime analysis of SSWM, quantifying its performance for various parameters and investigating differences to the \ea. We show that SSWM can have a moderate advantage over the \ea at crossing fitness valleys and study an example where SSWM outperforms the \ea by taking advantage of information on the fitness gradient.
\end{abstract}


\section{Introduction}
In the last 20 years evolutionary computation has developed a number of algorithmic techniques for the analysis of evolutionary and genetic algorithms. These methods typically focus on runtime, and allow for rigorous bounds on the time required to reach a global optimum, or other well-specified high-fitness solutions.
The runtime analysis of evolutionary algorithms has become one of the dominant concepts in evolutionary computation, leading to a plethora of results for evolutionary algorithms
~\cite{Auger2011,Jansen2013,NeumannWitt2010} as well as novel optimisation paradigms such as swarm intelligence~\cite{NeumannWitt2010} and artificial immune systems~\cite{Jansen2011a}.

Interestingly, although evolutionary algorithms are heavily inspired by natural evolution, these methods have seldom been applied to natural evolution as studied in mathematical population genetics. This is a missed opportunity: the time it takes for a natural population to reach a fitness peak is an important question for the study of natural evolution. The kinds of results obtained from runtime analysis, namely how the runtime scales with genome size and mutation rate, are of general interest to population genetics.
Moreover, recently there has been a renewed interest in applying computer science methods to problems in evolutionary biology with contributions from unlikely fields such as game theory~\cite{chastain_algorithms_2014}, machine learning~\cite{valiant_evolvability_2009} and Markov chain theory~\cite{chatterjee_time_2014}.
Here, we present a first attempt at applying runtime analysis to the so-called Strong Selection Weak Mutation regime of natural populations.

The Strong Selection Weak Mutation model applies when the population size, mutation rate, and selection strength are such that the time between occurrence of new mutations
is long compared to the time a new genotype takes to replace the parent genotype~\cite{gillespie_molecular_1984}.
Under these conditions, only one genotype is present in the population most of the time, and evolution occurs through ``jumps'' between different genotypes, corresponding to a new mutation replacing the resident genotype in the population. The relevant dynamics can then be characterized by a (1+1)-type stochastic process. This model is obtained as a limit of many other models, such as the Wright-Fisher model. One important aspect of this model is that new solutions are accepted with a probability $\frac{1-e^{-2\beta\Delta f}}{1-e^{-2 N\beta \Delta f}}$ that depends on the fitness difference $\Delta f$ between the new mutation and the resident genotype. Here $N$ reflects the size of the underlying population, and $\beta$ represents the selection strength. One can think of $f$ as defining a phenotype that is under selection to be maximized; $\beta$ quantifies how strongly a unit change in $f$ is favoured. This probability was first derived by Kimura~\cite{kimura_probability_1962} for a population of $N$ individuals that are sampled binomially in proportion to their fitness.

This choice of acceptance function introduces two main differences to the \ea: First, solutions of lower fitness (worsenings) may be accepted with some positive probability. This is reminiscent of the Metropolis algorithm (Simulated Annealing with constant temperature) which can also accept worsenings (see, e.\,g.~\cite{Jansen2007}).
Second, solutions of higher fitness can be rejected, since they are accepted with a probability that is roughly proportional to the relative advantage they have over the current solution.

We cast this model of natural evolution in a (1+1)-type algorithm referred to as SSWM, using common mutation operators from evolutionary algorithms. We then present first runtime analyses of this process.
Our aims are manifold:
\begin{itemize}
\item to explore the performance of natural evolution in the context of runtime, comparing it against simple evolutionary algorithms like the \ea,
\item to investigate the non-elitistic selection mechanism implicit to SSWM and its usefulness in the context of evolutionary algorithms, and
\item to show that techniques for the analysis of evolutionary algorithms can be applied to simple models of natural evolution, aiming to open up a new research field at the intersection of evolutionary computation and population genetics.
\end{itemize}

Our results are summarised as follows. For the simple function \onemax we show in Section~\ref{sec:onemax} that with suitably large population sizes, when $N \beta  \ge \frac{1}{2}\ln(11n)$, SSWM is an effective hill climber as it optimises \onemax in expected time $O((n \log n)/\beta)$. However, when the population size is by any constant factor smaller than this threshold, we encounter a phase transition and SSWM requires exponential time even on \onemax.

We then illustrate the particular features of the selection rule in more depth. In Section~\ref{sec:cliff} we consider a function $\cliff{d}$ where a fitness valley of Hamming distance~$d$ needs to be crossed. For $d = \omega(\log n)$ the \ea needs time $\Theta(n^d)$, but SSWM is faster by a factor of $e^{\Omega(d)}$ because of its ability to accept worse solutions. Finally, in Section~\ref{sec:balance} we illustrate on the function \balance~\cite{RohlfshagenLehreYao2009} that SSWM can drastically outperform the \ea because the fitness-dependent selection drives it to follow the steepest gradient. While the \ea needs exponential time in expectation, SSWM with overwhelming probability finds an optimum in polynomial time.

The main technical difficulties are that in contrast to the simple \ea, SSWM is a non-elitist algorithm, hence fitness-level arguments based on elitism are not applicable. Level-based theorems for non-elitist populations~\cite{Corus2014} are not applicable either because they require population sizes larger than~1. Moreover, while for the \ea transition probabilities to better solutions are solely determined by probabilities for flipping bits during mutation, for SSWM these additionally depend on the probability of fixation and hence the absolute fitness difference. The analysis of SSWM is more challenging than the analysis of the \ea, and requires tailored proof techniques. We hope that these techniques will be helpful for analysing other evolutionary algorithms with fitness-based selection schemes.

\section{Preliminaries}

We define the optimisation time of SSWM as the first generation where the optimum is accepted as new individual.

As can be seen from the description above, the model resembles the \ea in that it only maintains one genotype that may be replaced by mutated versions of it. The candidate solutions are accepted with probability
\begin{equation}
\pfix(\Delta f)=\frac{1-e^{-2\beta\Delta f}}{1-e^{-2 N\beta \Delta f}}
\end{equation}
where $\Delta f \neq 0$ is the fitness difference to the current solution and $N\geq 1$ is the size of the underlying population. For $\Delta f = 0$ we define $\pfix(0) := \lim_{\Delta f\rightarrow 0} \pfix(\Delta f)=\frac{1}{N}$, so that $\pfix$ is continuous and well defined for all $\Delta f$.
If $N=1$, this probability will be $p_\text{fix}(s)=1$, meaning that any offspring will be accepted, and if $N\rightarrow\infty$, it will only accept solution for which $\Delta f>0$.  This expression was first derived by Kimura~\cite{kimura_probability_1962} and represents the \emph{probability of fixation}, that is, the probability that a gene that is initially present in one copy in a population of $N$ individuals is eventually present in all individuals.

Since the acceptance function in this algorithm depends on the absolute difference in fitness between genotypes, we include a parameter $\beta\in (0,1]$  that effectively scales the fitness function and that in population genetics models the strength of selection on a phenotype. By incorporating $\beta$ as a parameter of this function (and hence of the algorithm) we avoid having to explicitly rescale the fitness functions we analyse, while allowing us to explore the performance of this algorithm on a family of functions. This function has a sigmoid shape (strictly increasing - see Lemma \ref{lemma:pfix-strictly-increasing}) with limits $\lim_{\Delta f\rightarrow -\infty}\pfix(\Delta f)=0$ and $\lim_{\Delta f\rightarrow \infty}\pfix(\Delta f)=1$. As such, for large $\vert\beta \Delta f\vert $ this probability of acceptance is close to the one in the \ea, as long as $N>1$, defeating the purpose of the comparison. By bounding $\beta$ to $1$, we avoid artefactual results obtained by inflating the fitness differences between genotypes.

We can then cast the SSWM regime as Algorithm~\ref{alg:sswm},  where the function $\textrm{mutate}(x)$ can be either standard bit mutation (all bits are mutated independently with probability $p_m=1/n$, which we call \emph{global mutations}) or flipping a single bit chosen uniformly at random (which we call \emph{local mutations}).
SSWM is valid when the expected number of new mutants in the population is much less than one, which implies that local mutations are a better approximation for this regime.  However, we also consider global mutations in order to facilitate a comparison with evolutionary algorithms such as the \ea (Algorithm~\ref{alg:ea}), which uses global mutations.

\begin{algorithm}[h]
\caption{SSWM}
\label{alg:sswm}
\begin{algorithmic}
	\STATE {Choose $x\in \{0,1\}^n$ uniformly at random }
	\REPEAT
		\STATE {$y\leftarrow \mathrm{mutate}(x) $}
		\STATE {$\Delta f =f(y)-f(x)$ }
		\STATE {Choose $r\in\left[0,1\right]$ uniformly at random}
		\IF {$r<\pfix (\Delta f)$}
			\STATE {$x\leftarrow y$}
		\ENDIF
	\UNTIL {stop}
\end{algorithmic}
\end{algorithm}

\begin{algorithm}[h]
\caption{\ea}
\label{alg:ea}
\begin{algorithmic}
	\STATE {Choose $x\in \{0,1\}^n$ uniformly at random }
	\REPEAT
		\STATE {$y\leftarrow \mathrm{mutate}(x)$}
		\IF {$f(y) \ge f(x)$}
			\STATE {$x\leftarrow y$}
		\ENDIF
	\UNTIL {stop}
\end{algorithmic}
\end{algorithm}

Next, we derive upper and lower bounds for $\pfix(\Delta f) $
that will be useful throughout the manuscript.

\begin{lemma}
For every $\beta \in \R^+$ and $N \in \N^+$ the following inequalities hold.
If $\Delta f \ge 0$ then
\[
\frac{2\beta \Delta f}{1+2\beta \Delta f} \le  \pfix(\Delta f) \le \frac{2\beta \Delta f}{1-e^{-2N\beta \Delta f}}.
\]
If $\Delta f \le 0$ then
\[
\frac{-2\beta\Delta f }{e^{-2 N \beta \Delta f}}\le  \pfix(\Delta f) \le \frac{e^{-2\beta\Delta f}}{e^{-2N\beta\Delta f}-1}.
\]

\label{lem:p.fix.bounds}
\end{lemma}

\begin{proof}
In the following we frequently use $1 + x \le e^x$ and $1-e^{-x}\leq 1$ for all $x \in \R$ as well as $e^x \le \frac{1}{1-x}$ for $x < 1$.

If $x\geq 0$,

\begin{eqnarray*}
	\pfix(\Delta f) &=& \frac{1-e^{-2\beta\Delta f}}{1-e^{-2N\beta\Delta f}} \geq  1-e^{-2\beta\Delta f}\\
				&\geq&  1-\frac{1}{1+2\beta\Delta f} = \frac{2\beta \Delta f}{1+2\beta \Delta f}
\end{eqnarray*}
as well as
\begin{eqnarray*}
	\pfix(\Delta f) &=& \frac{1-e^{-2\beta\Delta f}}{1-e^{-2N\beta\Delta f}}
				\leq \frac{2\beta \Delta f}{1-e^{-2N\beta\Delta f}}.
\end{eqnarray*}

If $\Delta f \leq 0$,
\begin{eqnarray*}
	\pfix(\Delta f) &=& 
				 \frac{e^{-2\beta\Delta f}-1}{e^{-2N\beta\Delta f}-1}
				\leq\frac{e^{-2\beta\Delta f}}{e^{-2N\beta\Delta f}-1}.
\end{eqnarray*}

Using the fact that $e^{-x}-1\leq e^{-x}$:
\begin{eqnarray*}
	\pfix(\Delta f) &=& 
	 \frac{e^{-2\beta\Delta f}-1}{e^{-2N\beta\Delta f}-1}
	\geq  \frac{e^{-2\beta\Delta f}-1}{e^{-2N\beta\Delta f}}
 \geq  \frac{-2\beta \Delta f }{e^{-2 N \beta \Delta f}}.
\end{eqnarray*}
\end{proof}

The previous bounds for $\Delta f>0$ show that \pfix \xspace is roughly proportional to the fitness difference between solutions $\beta \Delta f$.

\section{SSWM on OneMax}
\label{sec:onemax}

The function $\onemax(x) := \sum_{i=1}^n x_i$ has been studied extensively in natural computation because of its simplicity. It represents an easy hill climbing task, and it is the easiest function with a unique optimum for all evolutionary algorithms that only use standard bit mutation for variation~\cite{Sudholt2012c}. Showing that SSWM can optimise \onemax efficiently serves as proof of concept that SSWM is a reasonable optimiser. It further sheds light on how to set algorithmic parameters such as the selection strength~$\beta$ and the population size~$N$. To this effect, we first show a polynomial upper bound for the runtime of SSWM on \onemax. We then show that SSWM exhibits a phase transition on its runtime as a function of $N\beta$; changing this parameter by a constant factor leads to exponential runtimes on \onemax.

Another reason why studying \onemax for SSWM makes sense is because not all evolutionary algorithms that use a fitness-dependent selection perform well on \onemax. Oliveto and Witt~\cite{Oliveto2013a} showed that the Simple Genetic Algorithm, which uses fitness-proportional selection, fails badly on \onemax even within exponential time, with a very high probability.

\subsection{Upper Bound for SSWM on OneMax}

We first show the following simple lemma, which gives an upper bound on the probability of increasing or decreasing the number of ones in a search point by~$k$ in one mutation.
\begin{lemma}
\label{lem:mutations-decreasing-ones}
For any positive integer $k > 0$, let $\mut(i, i\pm k)$ for $0 \le i \le n$ be the probability that a global mutation of a search point with $i$ ones creates an offspring with $i\pm k$ ones. Then
\begin{align*}
\mut(i, i+k) & \le \left(\frac{n-i}{n}\right)^k \left(1-\frac{1}{n}\right)^{n-k} \cdot \frac{1.14}{k!} \\
\mut(i, i-k) & \le \left(\frac{i}{n}\right)^k \left(1-\frac{1}{n}\right)^{n-k} \cdot \frac{1.14}{k!}.
\end{align*}
\end{lemma}
The proof is omitted due to space restrictions; it uses arguments from the proof of Lemma~2 in~\cite{Sudholt2012c}. The second inequality follows immediately from the first one due to the symmetry $\mut(i,i-k)=\mut(n-i,n-i+k)$.

Now we introduce the concept of drift and find some bounds for its forward and backward expression.

\begin{definition}
Let $X_t$ be the number of ones in the current search point, for all $1\le i\le n$ the forward and backward drifts are
 \begin{align*}
  \Delta^{+}(i)=\;& E[X_{t+1}-i \mid X_{t}=i, X_{t+1}>i]\cdot P(X_{t+1}>i \mid X_{t}=i)\\
  \Delta^{-}(i)=\;& E[X_{t+1}-i \mid X_{t}=i, X_{t+1}<i]\cdot P(X_{t+1}<i \mid X_{t}=i)\\
 \shortintertext{and the net drift is the expected increase in the number of ones}
 \Delta (i) =\;& \Delta^+(i) + \Delta^-(i).
 \end{align*}
\end{definition}

\begin{lemma}
\label{lem:bounds.drift}
Consider SSWM on \onemax and mutation probability $p_{m}=\frac{1}{n}$. Then for global mutations, the forward and backward drifts can be bounded by
\begin{align*}
 \Delta^{+}(i) \ge\;& \frac{n-i}{n}\left(1-\frac{1}{n}\right)^{n-1} \pfix(1)\\
 |\Delta^{-}(i)| \le\;& 1.14\left( 1-\frac{1}{n}\right)^{n-1}\cdot \left(\pfix(-1) + e\cdot \pfix(-2)\right).
 \shortintertext{For local mutations the relations are as follows}
 \Delta^{+}(i) =\;& \frac{n-i}{n}\cdot \pfix(1)\\
 |\Delta^-(i)| \le\;& \pfix(-1).
 \end{align*}
\label{lem:drift.bounds}
\end{lemma}

\begin{proof}
For global mutations firstly we compute the lower bound for the forward drift,
\begin{equation}
\notag \Delta^{+}(i)=\sum_{j=1}^{n-i}\mut(i,i+j)\cdot j \cdot \pfix(j)\label{eq:Drift.pos}
\end{equation}
where $\mut(i,i+j)$ is the probability of mutation increasing the \Onemax value by $j$ and $i$ is the number of ones of the current search point.

\begin{align*}
 \Delta^{+}(i)\ge\;&   \mut(i,i+1)\cdot \pfix(1)\\
	\ge\;& \frac{n-i}{n}\left(1-\frac{1}{n}\right)^{n-1} \pfix(1).
 \end{align*}

Secondly we calculate the upper bound for the backward drift
\begin{align}
|\Delta^{-}(i)|=&\sum_{j=1}^{i}\mut(i,i-j)\cdot j \cdot \pfix(-j)
\notag \intertext{where $j$ is now the number of new zeros. We can upper bound $\mut(i,i-j)$ for the probability of flipping any $j$ bits, which from Lemma~\ref{lem:mutations-decreasing-ones} yields}\notag
\notag \le& \sum_{j=1}^{i}\frac{1.14}{j!}\cdot \left( 1-\frac{1}{n}\right)^{n-1}\cdot j\cdot \pfix(-j).
\notag \intertext{Separating the case $j=1$ and bounding the remaining fixation probabilities by $\pfix(-2)$} 
\notag \le&\; 1.14\left( 1-\frac{1}{n}\right)^{n-1}\pfix(-1) \\
\notag &+ 1.14\left(1-\frac{1}{n}\right)^{n-1} \cdot \sum_{j=2}^{i} \frac{1}{(j-1)!}\cdot \pfix(-2)\\
\notag \le&\; 1.14\left( 1-\frac{1}{n}\right)^{n-1} (\pfix(-1) + e\cdot \pfix(-2)).
\notag \end{align}
Finally, the case for local mutations is straightforward since the probability of a local mutation increasing the number of ones is $\frac{n-i}{n}$ and that of decreasing it is at most $1$.
\end{proof}

The following theorem shows that SSWM is efficient on \onemax for $N\beta \ge \frac{1}{2}\ln (11n)$, since then $\pfix(1)$ starts being greater than $n\cdot \pfix(-1)$ allowing for a positive drift even on the hardest fitness level ($n-1$ ones). The upper bound increases with $1/\beta$; this makes sense as for small values of~$\beta$ we have $\pfix(1) \approx 2\beta$ (cf.\ Lemma~\ref{lem:p.fix.bounds}). In this regime absolute fitness differences are small and improvements are only accepted with a small probability.

\begin{theorem}
\label{the:upper-onemax}
 For $N\beta \ge \frac{1}{2}\ln(11n)$ and $\beta \in (0,1]$, the expected optimisation time of SSWM on \onemax with local or global mutations is $O\left(\frac{n\log n}{\beta} \right)$ for every initial search point.
\end{theorem}

\begin{proof}

The fixation probabilities can be bounded as follows
 \begin{align}
  \pfix(1) &= \frac{1-e^{-2\beta}}{1-e^{-2N\beta}} \ge 1-e^{-2\beta} \notag
  \intertext{and for $N\beta \ge \frac{1}{2}\ln (11n)$}  
  \pfix(-1) &= \frac{e^{2\beta}-1}{e^{2N\beta}-1} \le \frac{e^{2\beta}-1}{11n-1}\label{eq:upper-bound-on-pfix-minus-one}\\
  \pfix(-2) &= \frac{e^{4\beta}-1}{e^{4N\beta}-1} \le \frac{e^{4\beta}-1}{(11n)^{2}-1} = O(n^{-2}). \notag
 \end{align}

Using Lemma~\ref{lem:bounds.drift}

\[
  \Delta(i) \ge \frac{1}{e}\left[ \frac{n-i}{n} \cdot (1-e^{-2\beta}) - 1.14\frac{e^{2\beta}-1}{11n-1} - O(n^{-2})\right]
\]

We need a positive net drift even in the last step $(n-i=1)$ 

\begin{align*}
\Delta(n-1) &\ge \frac{1}{e}\left[ \frac{1}{n} \cdot (1-e^{-2\beta}) - 1.14\frac{e^{2\beta}-1}{11n-1} - O(n^{-2})\right]  \\
            &\ge \frac{1}{e}\left[ \frac{1}{11n-1}\left( \frac{11n-1}{n} \cdot (1-e^{-2\beta}) - 1.14(e^{2\beta}-1) \right) \right] \\
            &\ge \frac{1}{e}\left[ \frac{1-e^{-2\beta}}{11n-1}\left( 11 -\frac{1}{n} - 1.14\cdot \frac{e^{2\beta}-1}{1-e^{-2\beta}}  \right) \right]
            \intertext{using the relation $e^x=\frac{e^{x}-1}{1-e^{-x}}$}
            &\ge \frac{1}{e}\left[ \frac{1-e^{-2\beta}}{11n-1}\left( 11 -\frac{1}{n} - 1.14\cdot e^{2\beta} \right) \right] \\
            \intertext{since $\beta \in (0,1]$ then $e^{2\beta}\le e^{2} < 7.5$}
            &\ge \frac{1}{e}\left[ \frac{1-e^{-2\beta}}{11n-1}\left( 2.5 -\frac{1}{n}  \right) \right] \\
            &\ge \frac{1.5}{e}\cdot \frac{1-e^{-2\beta}}{11n-1}
            \intertext{also for $\beta \in (0,1]$ we have $1.5(1-e^{-2\beta})\ge \beta$}
            &\ge \frac{\beta}{e}\cdot \frac{1}{11n-1}
\end{align*}
which is positive for enough large $n$.

Therefore we can lower bound the drift in any point as 

\begin{equation}
      \Delta(i) = \Omega\left(\frac{n-i}{n}\cdot \beta\right) \label{eq:onemax.drift}
\end{equation}

%

Now we apply Johannsen's variable drift theorem~\cite{Johannsen2010} to the number of zeros. Using ${h(z) := E(X_t-X_{t+1}\mid X_{t}=z)}$ then
\begin{align*}
E(T \mid X_0) &\leq \dfrac{z_{\min}}{h(z_{\min})}+ \int^{X_0}_{z_{\min}} \dfrac{1}{h(z)} dz
 \intertext{where $z$ is the number of zeros, $X_t$ the current state and $T$ the optimisation time. Introducing $z_{\min}=1$, $X_0=n$ and}
\notag \Delta(i)  &=\Omega\left(\frac{z}{n}\cdot \beta\right) =h(z)
\end{align*}
we obtain an upper bound for the runtime
\begin{align*}
 E(T \mid X_0) &\le \dfrac{1}{h(1)}+ \int^{n}_{1} \dfrac{1}{h(z)} dz = O\left( \frac{n}{\beta} \right) + O\left( \int^{n}_{1}   \frac{n}{\beta z}dz \right) \\
&= O\left( \frac{n}{\beta } (1+\log n)  \right) = O\left(\frac{n\log n}{\beta} \right).
\qedhere
\end{align*}
\end{proof}

\subsection{A Critical Threshold for SSWM on OneMax}

The upper bound from Theorem~\ref{the:upper-onemax} required $N \beta \ge \frac{1}{2} \ln(11n) = \frac{1}{2} \ln(n) + O(1)$. This condition is vital since if $N \beta$ is chosen too small, the runtime of SSWM on \onemax is exponential with very high probability, as we show next.

If $N \beta$ is by a factor of $1-\varepsilon$, for some constant~$\varepsilon > 0$, smaller than $\frac{1}{2} \ln n$, the  optimisation time is exponential in~$n$, with overwhelming probability. SSWM therefore exhibits a phase transition behaviour: changing $N \beta$ by a constant factor makes a difference between polynomial and exponential expected optimisation times on \onemax.
\begin{theorem}
\label{the:lower-onemax}
If $1 \le N\beta \le \frac{1-\varepsilon}{2} \ln n$ for some ${0 < \varepsilon < 1}$, then the optimisation time of SSWM with local or global mutations on \onemax is at least $2^{c n^{\varepsilon/2}}$ with probability ${1-2^{-\Omega(n^{\varepsilon/2})}}$, for some constant $c > 0$.
\end{theorem}

The idea behind the proof of Theorem~\ref{the:lower-onemax} is to show that for all search points with at least $n - n^{\varepsilon/2}$ ones, there is a negative drift for the number of ones. This is because for small $N \beta$ the selection pressure is too weak, and worsenings in fitness are more likely than steps where mutation leads the algorithm closer to the optimum.

We then use the negative drift theorem with self-loops presented in Rowe and Sudholt~\cite{Rowe2013} (an extension of the negative drift theorem~\cite{Oliveto2011} to stochastic processes with large self-loop probabilities). It is stated in the following for the sake of completeness.

\begin{theorem}[Negative drift with self-loops~\cite{Rowe2013}]
\label{thm:simplified-drift}
Consider a Markov process $X_0, X_1, \dots$ on $\{0, \dots, m\}$ and suppose there exists integers $a, b$ with $0 < a < b \leq m$ and $\varepsilon > 0$ such that for all $a \leq k \leq b$ the expected drift towards~0 is
\[
E(k - X_{t+1} \mid X_t = k) < -\varepsilon \cdot (1-p_{k,k})
\]
where $p_{k, k}$ is the self-loop probability at state~$k$.
Further assume there exists constants $r, \delta > 0$ (i.\,e. they are independent of $m$) such that for all $k \geq 1$ and all $d \ge 1$
\[
	p_{k, k - d}, p_{k, k + d} \leq \frac{r (1 - p_{k,k})}{(1 + \delta)^d}.
\]
Let $T$ be the first hitting time of a state at most~$a$, starting from $X_1 \geq b$. Let $\ell = b-a$. Then there is a constant $c > 0$ such that
\[
	\Prob{T \leq 2^{c \ell /r}} = 2^{-\Omega(\ell / r)}.
\]
\end{theorem}

The proof of Theorem~\ref{the:lower-onemax} applies Theorem~\ref{thm:simplified-drift} with respect to the number of zeros on an interval of $[0, n^{\varepsilon/2}]$.
\begin{proof}[Proof of Theorem~\ref{the:lower-onemax}]
We only give a proof for global mutations; the same analysis goes through for local mutations with similar, but simpler calculations.

Let $p_{k, j}$ be the probability that SSWM will make a transition from a search point with $k$ ones to one with $j$ ones. We start by pessimistically estimating transition probabilities and applying the negative drift theorem with regards to pessimistic transition probabilities $p_{k, j}'$ defined later. The drift theorem will be applied, taking the number of zeros as distance function to the optimum. Our notation refers to numbers of ones for simplicity. Throughout the remainder of the proof we assume $k \ge n - n^{\varepsilon/2}$.

From Lemma~\ref{lem:mutations-decreasing-ones} and every $1 \le j \le n-k$ we have
\begin{align}
p_{k, k+j} \le\;& \frac{1.14}{j!}  \cdot \left(\frac{n-k}{n} \right)^{j} \cdot  \left(1 - \frac{1}{n}\right)^{n-j} \cdot \pfix(j)\notag\\
\le\;& \frac{1.14}{j!} \cdot \left(\frac{n-k}{n}\right)^j \cdot  \pfix(j)\label{eq:bound-on-onemax-increase}\\
\le\;& 1.14\cdot \left(n^{\varepsilon/2-1}\right)^j \cdot \pfix(j)\notag.
\end{align}
Cf. Lemma \ref{lem:p.fix.bounds} we estimate $\pfix(j)$ by $\pfix(j) \le \frac{2\beta j}{1-e^{-2N \beta j}}$.
This gives
\[
p_{k, k+j} \le \left(n^{\varepsilon/2-1}\right)^j \cdot \frac{3\beta j}{1-e^{-2N\beta j}} := p_{k, k+j}'.
\]
The expected drift towards the optimum, $\Delta^+(k)$, is then bounded as follows
\begin{align*}
\Delta^+(k) \le\;& \sum_{j=1}^{n-k} j \cdot p_{k, k+j}'\\
\le\;& \sum_{j=1}^{n-k}  j \cdot \left(n^{\varepsilon/2-1}\right)^j \cdot \frac{3\beta j}{1-e^{-2N \beta j}}\\
\le\;& \frac{3\beta}{1-e^{-2N \beta}} \sum_{j=1}^{\infty} j^2 \cdot \left(n^{\varepsilon/2-1}\right)^j.\\
\intertext{Using $\sum_{j=1}^\infty j^2 \cdot x^j = \frac{x(1+x)}{(1-x)^3} \le x (1+ 5x)$ for $0 \le x \le 0.09$ (this holds for large enough~$n$ as $x= n^{\varepsilon/2-1} = o(1)$) as well as $N\beta  \ge 1$}
\le\;& \frac{3\beta}{1-e^{-2}} \cdot n^{\varepsilon/2-1} \cdot \left(1+5n^{\varepsilon/2-1}\right).
\end{align*}
On the other hand,
\begin{align*}
p_{k, k-1} \ge\;& \frac{k}{n} \cdot \left(1 - \frac{1}{n}\right)^{n-1} \cdot \pfix(-1)
\ge\; \frac{n-n^{\varepsilon/2}}{en} \cdot \pfix(-1)\\
=\;& \frac{\pfix(-1)}{e} \cdot \left(1 - n^{\varepsilon/2-1}\right)
\ge\; \frac{1}{e} \cdot \frac{2\beta}{e^{2N \beta}} \cdot \left(1 - n^{\varepsilon/2-1}\right)\\
\intertext{using $e^{2N \beta} \le e^{(1-\varepsilon)\ln n} = n^{1-\varepsilon}$}
\ge\;& \frac{2\beta \cdot n^{\varepsilon}}{en} \cdot \left(1 - n^{\varepsilon/2-1}\right) := p_{k, k-1}'.
\end{align*}
We further define $p_{k, k-j}' := 0$ for $j \ge 2$.
The expected increase in the number of ones at state~$k$, denoted $\Delta'(k)$, with regards to the pessimistic Markov chain defined by $p_{k, j}'$ is hence at most
\begin{align*}
& \Delta'(k)\\
\le & \sum_{j=1}^{n-k} j \cdot p_{k, k+j}' - p_{k, k-1}'\\
\le\;& \frac{3\beta}{1-e^{-2}} \cdot n^{\varepsilon/2-1} \cdot \left(1+5n^{\varepsilon/2-1}\right) - \frac{2\beta \cdot n^{\varepsilon}}{en} \cdot \left(1 - n^{\varepsilon/2-1}\right)\\
=\;& 2\beta \cdot n^{\varepsilon/2-1} \cdot \left( \frac{3}{2}\cdot \frac{1+5n^{\varepsilon/2-1}}{1-e^{-2}} - \frac{n^{\varepsilon/2}}{e} \cdot \left(1 - n^{\varepsilon/2-1}\right)\right)\\
=\;& -\Omega(\beta \cdot n^{\varepsilon-1}).
\end{align*}

Now, the self-loop probability for the pessimistic Markov chain is at least $p_{k, k}' \ge 1 - \sum_{j =1}^{n-k} p_{k, k+j}' - p_{k, k-1}' \ge 1 - \sum_{j =1}^{n-k} j \cdot p_{k, k+j}' - p_{k, k-1}' = 1 - O(\beta n^{\varepsilon-1})$, hence the first condition of the drift theorem is satisfied.

The second condition on exponentially decreasing transition probabilities follows from $p_{k, k-1}' \le 1-p_{k,k}'$, $p_{k, k-j}' = 0$ for $j \ge 2$ and, for all $j \in \N$,
\begin{align*}
p_{k, k+j}' =\;& \left(n^{\varepsilon/2-1}\right)^j \cdot \frac{3\beta j}{1-e^{-2N \beta j}}
\le\; \left(n^{\varepsilon/2-1}\right)^j \cdot \frac{3\beta j}{1-e^{-2}}\\
\intertext{multiplying by $p_{k,k-1}'/p_{k,k-1}'$}
=\;& p_{k, k-1}' \cdot \frac{\left(n^{\varepsilon/2-1}\right)^j \cdot \frac{3\beta j}{1-e^{-2}}}{\frac{2\beta \cdot n^{\varepsilon}}{en} \cdot \left(1 - n^{\varepsilon/2-1}\right)}\\
=\;& p_{k, k-1}' \cdot n^{-\varepsilon/2} \cdot \left(n^{\varepsilon/2-1}\right)^{j-1} \cdot \frac{e}{1-n^{\varepsilon/2-1}} \cdot \frac{3}{2}\cdot \frac{j}{1-e^{-2}}\\
\le\;& p_{k, k-1}' \cdot 2^{-j}
\le\; (1-p_{k, k}') \cdot 2^{-j}
\end{align*}
where the penultimate inequality holds for large enough~$n$. This proves the second condition for $\delta := 1$ and $r := 2$. Now the negative drift theorem, applied to the number of zeros, proves the claimed result.
\end{proof}

\section{On Traversing Fitness Valleys}
\label{sec:cliff}

We have shown that with the right parameters, SSWM is an efficient hill climber. On the other hand, in contrast to the \ea, SSWM can accept worse solutions with a probability that depends on the magnitude of the fitness decrease. This is reminiscent of the Metropolis algorithm---although the latter accepts every improvement with probability~1, whereas SSWM may reject improvements.

Jansen and Wegener~\cite{Jansen2007} compared the ability of the \ea and a Metropolis algorithm in crossing fitness valleys and found that both showed similar performance on \emph{smooth integer} functions: functions where two Hamming neighbours have a fitness difference of at most~1~\cite[Section~6]{Jansen2007}.

We consider a similar function, generalising a construction by J{\"a}gersk{\"u}pper and Storch~\cite{Jagerskupper2007a}: the function $\cliff{d}$ is defined such that non-elitist algorithms have a chance to jump down a ``cliff'' of height roughly~$d$ and to traverse a fitness valley of Hamming distance~$d$ to the optimum (see Figure \ref{fig-cliff}). 

\usetikzlibrary{calc}
\usetikzlibrary{mindmap}
\usetikzlibrary{arrows}
\usetikzlibrary{shadows}
\usetikzlibrary{plotmarks}
\usetikzlibrary{backgrounds}
\usetikzlibrary{shapes}
\usetikzlibrary{shapes.symbols}

%

\begin{figure}[H]
\begin{center}
\begin{tikzpicture}[ domain=0:30,xscale=0.23,yscale=0.15, scale=1.2, every
shadow/.style={shadow xshift=0.0mm, shadow yshift=0.4mm}]
\tikzstyle{helpline}=[black,thick];
\tikzstyle{function}=[blue, thick];
\draw[gray!50,line width=0.2pt,xstep=1,ystep=1] (0,0) grid (20.5,15.5);
\draw[function] (0,0) -- (12,12) -- (13,7.5) -- (18,12.5);
\draw[helpline, thin, -triangle 45] (0,0) -- (0,15.5);
\draw[helpline, thin, -triangle 45] (0,0) -- (20.5,0) node[right] {$\left| x\right|_1$}; 
\draw[helpline] (0,0) -- (0,15.5);
\draw[helpline] (0,0) -- (20.5,0);

\draw[helpline] (0,0) -- (-1,0) node[left] {$0$};
\draw[helpline] (0,0) -- (0,-1) node[below] {$0$};
\draw[helpline] (0,15) -- (0,15) node[above=5pt] {fitness};
\draw[helpline] (12,0) -- (12,-1) node[below] {$n-d$};
\draw[helpline] (18,0) -- (18,-1) node[below] {$n$};

 \foreach \u in {0,...,12} {
       \filldraw (\u, \u) circle (3.8pt);
   }

 \foreach \u in {13,...,18} {
       \filldraw (\u, \u - 5.5) circle (3.8pt);
   }

\end{tikzpicture}
\end{center}
\vskip -2em
\caption{Sketch of the function $\cliff{d}$.}  \label{fig-cliff}
\label{fig:cliff}
\end{figure}
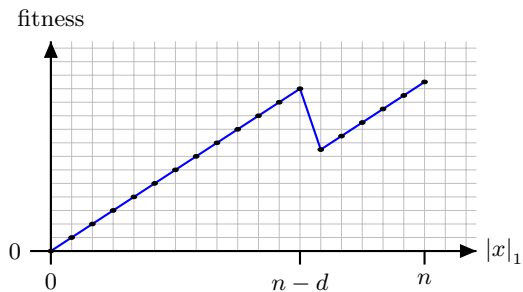
 \vskip -2em

\begin{definition}[Cliff]\label{def-cliff}
\[
\cliff{d}(x) = \begin{cases}  \ones{x}  &\mbox{if } \ones{x} \leq n-d \\
\ones{x}-d+\frac{1}{2}  & \mbox{otherwise}
\end{cases}
\]
where $\ones{x} = \sum_{i=1}^n x_i$   counts the number of ones.
\end{definition}

The \ea typically optimises $\cliff{d}$ through a direct jump from the top of the cliff to the optimum, which takes expected time $\Theta(n^d)$.
\begin{theorem}
\label{the:ea-on-cliff}
The expected optimisation time of the \ea on $\cliff{d}$, for $2 \le d \le n/2$, is $\Theta(n^d)$.
\end{theorem}

In order to prove Theorem~\ref{the:ea-on-cliff}, the following lemma will be useful for showing that the top of the cliff is reached with good probability. More generally, it shows that the conditional probability of increasing the number of ones in a search point to~$j$, given it is increased to some value of~$j$ or higher, is at least $1/2$.
\begin{lemma}
\label{lem:conditional-mut}
For all $0 \le i < j \le n$,
\[
\frac{\mut(i, j)}{\sum_{k=j}^n \mut(i, k)} \ge \frac{1}{2}.
\]
\end{lemma}
The proof of this lemma is presented in the appendix.

\begin{proof}[Proof of Theorem~\ref{the:ea-on-cliff}]
From any search point with $i < n-d$ ones, the probability of reaching a search point with higher fitness is at least $\frac{n-i}{en}$. The expected time for accepting a search point with at least $n-d$ ones is at most $\sum_{i=0}^{n-d-1} \frac{en}{n-i} = O(n \log n)$. Note that this is $O(n^d)$ since $d \ge 2$.

We claim that with probability $\Omega(1)$, the first such search point has $n-d$ ones: with probability at least $1/2$ the initial search point will have at most $n-d$ ones. Invoking Lemma~\ref{lem:conditional-mut} with $j := n-d$, with probability at least $1/2$ the top of the cliff is reached before any other search point with at least $n-d$ ones.

Once on the top of the cliff the algorithm has to jump directly to the optimum to overcome it.  The probability of such a jump is
$\frac{1}{n^d} (1-\frac{1}{n})^{n-d}$ and therefore the expected time to make this jump is $\Theta(n^d)$.
\end{proof}

SSWM with global mutations also has an opportunity to make a direct jump to the optimum. However, compared to the \ea its performance slightly improves when considering shorter jumps and accepting a search point of inferior fitness.
The following theorem shows that for large enough cliffs, $d = \omega(\log n)$, the expected optimisation time is by a factor of $e^{\Omega(d)}$ smaller than that of the \ea. Although both algorithms need a long time for large~$d$, the speedup of SSWM is significant for large~$d$.
\begin{theorem}
The expected optimisation time of SSWM with global mutations and $\beta=1, N = \frac{1}{2}\ln(\const n)$ on $\cliff{d}$ with $d = \omega(\log n)$ is at most~$n^{d}/e^{\Omega(d)}$.
\end{theorem}
\begin{proof}
We define $R$ as the expected time for reaching a search point with either $n-d$ or $n$ ones, when starting with a worst possible non-optimal search point. Let $\Tpeak$ be the random optimisation time when starting with any search point of $n-d$ ones, hereinafter called a \emph{peak}. Then the expected optimisation time from any initial point is at most $R + \E{\Tpeak}$. Let $\psuccess$ be the probability of SSWM starting in a peak will reach the optimum before reaching a peak again. We call such a time period a \emph{trial}. After the end of a trial, taking at most $R$ expected generations, with probability $1-\psuccess$ SSWM returns to a peak again, so
\begin{align}
& \E{\Tpeak} \le R + (1-\psuccess) \cdot \E{\Tpeak}\notag\\
\Leftrightarrow\; & \E{\Tpeak} \le \frac{R}{\psuccess}.\label{eq:Tpeak-recurrence}
\end{align}
We first bound the worst-case time to return to a peak or a global optimum as $R = O(n \log n)$. Let $S_1$ be the set of all search points with at most $n-d$ ones and ${S_2 := \{0, 1\}^n \setminus S_1}$. As long as the current search point remains within $S_2$, SSWM essentially behaves like on \Onemax. Repeating arguments from the proof of Theorem~\ref{the:upper-onemax}, in expected time $O((n \log n)/\beta) = O(n \log n)$ (as here $\beta=1$) SSWM either finds a global optimum or a search point in~$S_1$. Likewise, as long as the current search point remains within $S_1$, SSWM essentially behaves like on \Onemax and within expected time $O(n \log n)$ either a peak or a search point in $S_2$ is found.

SSWM can switch indefinitely between $S_1$ and $S_2$ within one trial, as long as no optimum or peak is reached. The conditional probability of creating a peak---when from a search point with $i < n-d$ ones either a peak or a non-optimal search point in $S_2$ is reached---is
\begin{align*}
\frac{\mut(i, n-d) \cdot \pfix(n-d-i)}{\sum_{k=n-d}^{n-1} \mut(i, k) \cdot \pfix(k-i-d+1/2)} \ge\;& \frac{\mut(i, j)}{\sum_{k=j}^n \mut(i, k)}
\end{align*}
as $\pfix(n-d-i) \ge \pfix(k-i-d+1/2)$ for all $n-d < k < n$. By Lemma~\ref{lem:conditional-mut}, the above fraction is at least~$1/2$.
Hence SSWM in expectation only makes $O(1)$ transitions from $S_1$ to $S_2$, and the overall expected time spent in $S_1$ and $S_2$ is at most $R = O(1) \cdot O(n \log n)$.

The remainder of the proof now shows a lower bound on $\psuccess$, the probability of a trial being successful. A sufficient condition for a successful trial is that the next mutation creates a search point with $n-d+k$ ones, for some integer $1 \le k \le d$, that this point is accepted, and that from there the global optimum is reached before returning to a peak.

We estimate the probabilities for these events separately in order to get an overall lower bound on the probability of a trial being successful.

From any peak there are $\binom{d}{k}$ search points at Hamming distance~$k$ that have $n-d+k$ ones. Considering only such mutations, the probability of a mutation increasing the number of ones from $n-d$ by~$k$ is at least
\begin{align*}
\mut(n-d, n-d+k) \ge\;& \frac{1}{n^k} \cdot \left(1 - \frac{1}{n}\right)^{n-1} \cdot \binom{d}{k}\\
\ge\;& \frac{1}{en^k} \cdot \left(\frac{d}{k}\right)^k.
\end{align*}
The probability of accepting such a move is
\[
\pfix(k-d+1/2) = \frac{e^{2\beta(d-k-1/2)}-1}{e^{2N \beta(d-k-1/2)}-1} \ge
\frac{e^{2(d-k-1/2)}-1}{(\const n)^{(d-k-1/2)}}.
\]
We now fix $k := \lfloor d/e\rfloor$ and estimate the probability of making and accepting a jump of length~$k$:
\begin{align*}
& \mut(n-d, n-d+k) \cdot \pfix(k-d+1/2)\\
\;&\ge
\frac{1}{en^k} \cdot \left(\frac{d}{k}\right)^k \cdot
\frac{e^{2(d-k-1/2)}-1}{(\const n)^{(d-k-1/2)}}\\
\;&=
\Omega\left(n^{-d+1/2} \cdot \left(\frac{d}{k}\right)^k \cdot
\left(\frac{e^2}{\const}\right)^{d-k}\right)\\
\;&=
\Omega\left(n^{-d+1/2} \cdot \left({e^{1/e}} \cdot \left(\frac{e^2}{\const}\right)^{1-1/e}\right)^{d}\right)\\
\;&=
\Omega\left(n^{-d+1/2} \cdot \left(\frac{10}{9}\right)^{d}\right).
\end{align*}
Finally, we show that, if SSWM does make this accepted jump, with high probability it climbs up to the global optimum before returning to a search point in $S_1$. To this end we work towards applying the negative drift theorem to the number of ones in the interval $[a := \lceil n-d + k/2 \rceil, b := n-d+k]$ and show that, since we start in state~$b$, a state $a$ or less is unlikely to be reached in polynomial time.

We first show that the drift is typically equal to that on \Onemax. For every search point with more than $a$ ones, in order to reach $S_1$, at least $k/2$ bits have to flip. Until this happens, SSWM behaves like on \onemax and hence reaches either a global optimum or a point in $S_1$ in expected time $O(n \log n)$.
The probability for a mutation flipping at least $k/2$ bits is at most $1/((\ln n)/(2e))! = n^{-\Omega(\log n)}$, so the probability that this happens in expected time $O(n \log n)$ is still $n^{-\Omega(\log n)}$.

Assuming such jumps do not occur, we can then use drift bounds from the analysis of \onemax for states with at least $a$ ones. From the proof of Theorem~\ref{the:upper-onemax} and \eqref{eq:onemax.drift} we know that the drift at $i$ ones for $\beta=1$ is at least
\[
\Delta(i) \ge \Omega\left(\frac{n-i}{n}\right).
\]
Let $p_{i, j}$ denote the transition probability from a state with $i$ ones to one with $j$ ones. The probability of decreasing the current state is at most $\pfix(-1) = O(1/n)$ due to~\eqref{eq:upper-bound-on-pfix-minus-one}. The probability of increasing the current state is at most $(n-i)/n$ as a necessary condition is that one out of $n-i$ zeros needs to flip. Hence for $i \le b$, which implies $n-i = \omega(1)$, the self-loop probability is at least
\[
p_{i, i} \ge 1 - O\left(\frac{1}{n}\right) - \frac{n-i}{n} = 1 - O\left(\frac{n-i}{n}\right).
\]
Together, we get $\Delta(i) \ge \Omega(1-p_{i, i})$, establishing the first condition of Theorem~\ref{thm:simplified-drift}.

Note that $\pfix(1) = \frac{1-e^{-2}}{1-1/n} = \Omega(1)$, hence
\begin{equation}
\label{eq:bound-on-1-pii}
1 - p_{i, i} \ge p_{i, i+1} \ge \frac{n-i}{en} \cdot \pfix(1) = \Omega\left(\frac{n-i}{n}\right).
\end{equation}
The second condition follows for improving jumps from $i$ to $i+j$, $j \ge 1$, from Lemma~\ref{lem:mutations-decreasing-ones} and~\eqref{eq:bound-on-1-pii}:
\[
p_{i, i+j} \le \left(\frac{n-i}{n}\right)^j \cdot \frac{1}{j!} \cdot \pfix(j)
\le \frac{n-i}{n} \cdot \frac{1}{j!}
\le (1-p_{i, i}) \cdot \frac{O(1)}{2^j}.
\]
For backward jumps we get, for $1 \le j \le k/2$, and $n$ large enough,
\[
p_{i, i-j} \le \pfix(-j) \le \frac{e^{2j}}{e^{2Nj}-1} =
\frac{e^{2j}}{(\const n)^{j}-1} \le 2^{-j}.
\]
Now Theorem~\ref{thm:simplified-drift} can be applied with $r = O(1)$ and $\delta = 1$ and it yields that the probability of reaching a state of $a$ or less in $n^{\omega(1)}$ steps is $n^{-\omega(1)}$.

This implies that following a length-$k$ jump, a trial is successful with probability $1-n^{-\omega(1)}$. This establishes
$\psuccess := \Omega\left(n^{-d+1/2} \cdot \left(\frac{10}{9}\right)^{d}\right)$. Plugging this into~\eqref{eq:Tpeak-recurrence}, adding time $R$ for the time to reach the peak initially, and using that $O(n^{1/2}\log n) \cdot (9/10)^d = e^{-\Omega(d)}$ for $d = \omega(\log n)$ yields the claimed bound.
\end{proof}

\section{SSWM Outperforms (1+1)~EA on Balance}
\label{sec:balance}

Finally, we investigate a feature that distinguishes SSWM from the \ea as well as the Metropolis algorithm: the fact that larger improvements are more likely to be accepted than smaller improvements.

To this end, we consider the function \balance, originally introduced by Rohlfshagen, Lehre, and Yao~\cite{RohlfshagenLehreYao2009} as an example where rapid dynamic changes in dynamic optimisation can be beneficial.
The function has also been studied in the context of stochastic ageing by Oliveto and Sudholt~\cite{Oliveto2014}.

In its static (non-dynamic) form, \balance can be illustrated by a two-dimensional plane, whose coordinates are determined by the number of leading ones (LO) in the first half of the bit string, and the number of ones in the second half, respectively. The former has a steeper gradient than the latter, as the leading ones part is weighted by a factor of~$n$ in the fitness (see Figure \ref{fig-balance}).

\begin{definition}[\balance~\cite{RohlfshagenLehreYao2009}]\label{def-balance}
Let $a,b \in \{0,1\}^{n/2}$ and $x = ab \in \{0,1\}^n$. Then, $\balance(x) = $
\begin{equation*}
	\begin{cases}
		n^3 & \text{if } \lo{a} = n/2, \text{else}\\
		\ones{b} + n \cdot \lo{a} & \text{if } n/16 < \ones{b} < 7n/16 , \text{else}\\
		n^2 \cdot \lo{a} & \text{if }  \zeros{a} > \sqrt{n}, \text{else}\\
		0 & \text{otherwise.}
	\end{cases}
\end{equation*}
where $\ones{x} = \sum_{i=1}^{n/2} x_i$,  $\zeros{x} $ is a number of zeros and $\lo{x} :=\sum_{i=1}^{n/2} \prod_{j=1}^i x_j$ counts the number of leading ones.
\end{definition}

\begin{figure}
	\center
	\begin{tikzpicture}[scale=0.4]
		\draw (0,0) rectangle (10,5);
		\draw (0,1) -- (9,1);
		\draw (0,4) -- (9,4);
		\draw (9,0) -- (9,5);
		\draw (8,0) -- (8,1);
		\draw (8,4) -- (8,5);
		
		\node at (8.5,0.5) {0};
		\node at (8.5,4.5) {0};
		\node at (9.5,2.5) {$n^3$};
		\node at (4,0.5) {$n^2 \cdot \lo{a}$};
		\node at (4,4.5) {$n^2 \cdot \lo{a}$};
		\node at (4,2.5) {$n \cdot \lo{a} + \ones{b}$};
		
		\draw[->] (0,-0.5) -- node[below]{$\lo{a}$} (10,-0.5);
		\draw[->] (-0.5,0) --node[left]{$\ones{b}$} (-0.5,5);
		
	\end{tikzpicture}
	\caption{\boldmath Visualisation of \balance\cite{RohlfshagenLehreYao2009}.\label{fig-balance}}
\end{figure}

The function is constructed in such a way that all points with a maximum number of leading ones are global optima, whereas increasing the number of ones in the second half beyond a threshold of $7n/16$ (or decreasing it below a symmetric threshold of $n/16$) leads to a trap, a region of local optima that is hard to escape from.

Rohlfshagen, Lehre, and Yao~\cite[Theorem~3]{RohlfshagenLehreYao2009} showed the following lower bound for the \ea, specialised to non-dynamic optimisation:
\begin{theorem}[\cite{RohlfshagenLehreYao2009}]
The expected optimisation time of the \ea on \balance is $n^{\Omega(n^{1/2})}$.
\end{theorem}

We next show that SSWM with high probability finds an optimum in polynomial time.
For appropriately small~$\beta$ we have sufficiently many successes on the LO-part such that we find an optimum before the \onemax -part reaches the region of local optima. This is because for small~$\beta$ the probability of accepting small improvements is small. The fact that SSWM is slower than the \ea on \Onemax by a factor of $O(1/\beta)$ turns into an advantage over the \ea on \balance.

The following lemma shows that SSWM effectively uses elitist selection on the LO-part of the function in a sense that every decrease is rejected, with overwhelming probability.
\begin{lemma}
\label{lem:lo-elitism}
For every~$x = ab$ with $n/16 < \ones{b} < 7n/16$ and $\beta = n^{-3/2}$ and $N \beta = \ln n$, the probability of SSWM with local or global mutations accepting a mutant~$x'=a'b'$ with $\lo{a'} < \lo{a}$ and $n/16 < \ones{b'} < 7n/16$ is $O(n^{-n})$.
\end{lemma}
\begin{proof}
The loss in fitness is at least $n-(\ones{b'}-\ones{b}) \ge n/2$.
The probability of SSWM accepting such a loss is at most
\begin{align}
\notag \pfix(-n/2) & \le  \frac{1-e^{-2\beta(-n/2)}}{1-e^{-2N \beta(-n/2)}}
\le \frac{e^{2\beta(n/2)}}{e^{2N \beta(n/2)}-1}.
\end{align}

Assuming $\beta = n^{-3/2}$ and $N \beta = \ln n$, this is at most
\begin{align*}
 \frac{e^{\frac{\sqrt{n}}{n}}}{n^{n}-1}
 \le \frac{e}{n^{n}-1}
= O(n^{-n}).
\qquad\qedhere
\end{align*}
\end{proof}

The following lemma establishes the optimisation time of the SSWM algorithm on either the \onemax or the LO-part of \balance.

For global mutations we restrict our considerations to \emph{relevant steps}, defined as steps where no leading ones in the first half of the bit string is flipped. The probability of a relevant step is always at least $(1-1/n)^{n/2} \approx e^{-1/2}$. When using local mutations, all steps are defined as relevant.
\begin{lemma}
\label{lem:time-on-lo-part}
Let $\beta = n^{-3/2}$ and $N \beta = \ln n$. With probability ${1-e^{-\Omega(n^{1/2})}}$, SSWM with either local or global mutations either optimises the LO part or reaches the trap (all search points with fitness $n^2 \cdot \lo{a}$) within
\[
T := \frac{n^2}{4} \cdot \frac{1}{\pfix(n-\sqrt{n})} \cdot \left(1 + n^{-1/4}\right)
\]
relevant steps.
\end{lemma}
\begin{proof}
Consider a relevant step, implying that global mutations will leave all leading ones intact.
With probability $1/n$ a local or global mutation will flip the first 0-bit. This increases the fitness by $k \cdot n - \Delta_{\mathrm{OM}}$, where $\Delta_{\mathrm{OM}}$ is the difference in the \onemax-value of~$b$ caused by this mutation and $k$ is the number of consecutive 1-bits following this bit position after mutation. The latter bits are called \emph{free riders} and it is well known (see~\cite[Lemma~1 and proof of Theorem~2]{Lehre2012}) that the number of free riders follows a geometric distribution with parameter~$1/2$, only capped by the number of bits to the end of the bit string~$a$.

The probability of flipping at least $\sqrt{n}$ bits in one global mutation is at most $1/(\sqrt{n})! = e^{-\Omega(\sqrt{n})}$ and the probability that this happens at least once in $T$ relevant steps is still of the same order (using that $T = \poly{n}$ as $\pfix(n-\sqrt{n}) \ge 1/N \ge 1/\poly{n}$). We assume in the following that this does not happen, which allows us to assume $\Delta_{\mathrm{OM}} \le \sqrt{n}$. We also assume that the number of leading ones is never decreased during non-relevant steps as the probability of accepting such a fitness decrease is $O(n^{-n})$ by Lemma~\ref{lem:lo-elitism} and the expected number of non-relevant steps before $T$ relevant steps have occurred is $O(T)$.

We now have that the number of leading ones can never decrease and any increase by mutation is accepted with probability at least $\pfix(n-\sqrt{n})$. In a relevant step, the probability of increasing the number of leading ones is hence at least $1/n \cdot \pfix(n-\sqrt{n})$ and the expected number of such improvements in
\[
T := \frac{n^2}{4} \cdot \frac{1}{\pfix(n-\sqrt{n})} \cdot (1+n^{-1/4})
\]
relevant steps is at least $n/4 + n^{3/4}/4$.
By Chernoff bounds~\cite{Doerr2011chapter}, the probability that less than $n/4 + n^{3/4}/8$ improvements happen is $e^{-\Omega(n^{1/2})}$. Also the probability that during this number of improvements less than $n/4 - n^{3/4}/8$ free riders occur is $e^{-\Omega(n^{1/2})}$. If these two rare events do not happen, a LO-value of $n/2$ is reached before time~$T$. Taking the union bound over all rare failure probabilities proves the claim.
\end{proof}

We now show that the \onemax part is not optimized before the LO part.
\begin{lemma}
\label{lem:time-on-om-part}
Let $\beta = n^{-3/2}$, $N \beta = \ln n$, and $T$ be as in Lemma~\ref{lem:time-on-lo-part}. The probability that SSWM starting with $a_0b_0$ such that $n/4 \le \ones{b_0} \le n/4 + n^{3/4}$ creates a search point $ab$ with $\ones{b} \le n/16$ or $\ones{b} \ge 7n/16$ in~$T$ relevant steps is $e^{-\Omega(n^{1/2})}$.
\end{lemma}
It will become obvious that in $T$ relevant steps SSWM typically makes a progress of $O(n)$ on the \onemax part.
The proof of Lemma~\ref{lem:time-on-om-part} requires a careful and delicate analysis to show that the constant factors are small enough such that the stated thresholds for $\ones{b}$ are not surpassed.
\begin{proof}[Proof of Lemma~\ref{lem:time-on-om-part}]
We only prove that a search point with $\ones{b} \ge 7n/16$ is unlikely to be reached with the claimed probability. The probability for reaching a search point with $\ones{b} \le n/16$ is clearly no larger, and a union bound for these two events leads to a factor of 2 absorbed in the asymptotic notation.

Note that for $\beta = n^{-3/2}$ we have
\[
\pfix(n-\sqrt{n}) \ge \frac{2\beta(n-\sqrt{n})}{1+2\beta(n-\sqrt{n})}
\ge
2\beta n \cdot (1-O(n^{-1/2})).
\]
Hence
\[
T \le \frac{n^2}{4} \cdot \frac{1}{2\beta n} \cdot \left(1 + O(n^{-1/2})\right)
= \frac{n}{8\beta} \cdot \left(1 + O(n^{-1/2})\right).
\]
We call a relevant step \emph{improving} if the number of ones in~$b$ increases and the step is accepted.

We first consider only steps where the number of leading ones stays the same.
Then the probability that the \onemax value increases from~$k$ by~$j$, adapting Lemma~\ref{lem:mutations-decreasing-ones} to a string of length~$n/2$, is at most
\begin{align*}
p_{j} \le\;& \left(\frac{n/2-k}{n}\right)^j \cdot \frac{1.14}{j!} \cdot \pfix(j)\\
\intertext{using $n/2 - k \le n/4$}
\le\;& \frac{1.14 \cdot 4^{-j}}{j!} \cdot \pfix(j) \leq \frac{1.14 \cdot 4^{-j}}{j!} \cdot \frac{2\beta j}{1-e^{-2N \beta j}}\\
\le\;& 2.28\beta \cdot 4^{-j} \cdot \frac{1}{1-e^{-2N \beta j}} =: p_j.
\end{align*}
In the following, we work with pessimistic transition probabilities~$p_j$.
Note that for all $j \ge 1$
\begin{align*}
\frac{p_j}{p_1} = 4^{-(j-1)} \cdot \frac{1-e^{-2N \beta}}{1-e^{-2N \beta j}}
\le 4^{-(j-1)}.
\end{align*}
Let $p^+$ denote (a lower bound on) the probability of an improving step, then
\begin{align*}
p^+ \le\;& \sum_{j=1}^{\infty} p_j \le p_1 \cdot \sum_{j=1}^\infty 4^{-(j-1)} = p_1 \cdot \frac{4}{3}.
\end{align*}
The conditional probability of advancing by~$j$, given an improving step, is then
\begin{align*}
\frac{p_j}{p^+} \le 4^{-(j-1)} \cdot \frac{p_1}{p^+} = \left(1 - \frac{3}{4}\right)^{j-1} \cdot \frac{3}{4},
\end{align*}
which corresponds to a geometric distribution with parameter~$3/4$.

Now, by Chernoff bounds, the probability of having more than $S := (1+n^{-1/4}) \cdot p^+ \cdot T$ improving steps in $T$ relevant steps is $e^{-\Omega(n^{1/2})}$. Using a Chernoff bound for geometric random variables~\cite[Theorem~1.14]{Doerr2011chapter}, the probability of $S$ improving steps yielding a total progress of at least ${(1+n^{-1/4}) \cdot 4/3 \cdot S}$ is $e^{-\Omega(n^{1/2})}$.

If none of these rare events happen, the progress is at most
\begin{align*}
& (1+O(n^{-1/4})) \cdot \frac{4}{3} \cdot p^+ \cdot T\\
=\;& (1+O(n^{-1/4})) \cdot \frac{16}{9} \cdot p_1 \cdot T\\
\le\;& (1+O(n^{-1/4})) \cdot \frac{1.14}{9} \cdot n.
\end{align*}

We also have at most $n/2$ steps where the number of leading ones increases. If the number of leading ones increases by $\delta \ge 1$, the fitness increase is $\delta n + \ones{b'} - \ones{b}$. Hence the above estimations of jump lengths are not applicable.
We call these \emph{special} steps; they are unorthodox as the large fitness increase makes it likely that any mutation on the \onemax part is accepted. We show that the progress on the \onemax part across all special steps is $O(n^{3/4})$ with high probability.

We grant the algorithm an advantage if we assume that, after initialising with $\ones{b} \ge n/4$, no search point with $\ones{b} < n/4$ is ever reached\footnote{Otherwise, we restart our considerations from the first point in time where $\ones{b} \ge n/4$ again, replacing $T$ with the number of remaining steps. With overwhelming probability we will then again have $\ones{b} \le n/4 + n^{3/4}$.}. Under this assumption we always have at least as many 1-bits as 0-bits in $b$, and mutation in expectation flips at least as many 1-bits to 0 as 0-bits to 1.

Then the progress in $\ones{b}$ in one special step increasing the number of leading ones by $\loincrease$ can be described as follows. Imagine a matching (pairing) between all bits in $b$ such that each pair contains at least one 1-bit. Let $X_i$ denote the random change in $\ones{b}$ by the $i$-th pair. If the pair has two 1-bits, $X_i \le 0$ with probability~1. Otherwise, we have $X_i = 1$ if the 0-bit in the pair is flipped, the 1-bit in the pair is not flipped, and the mutant is accepted (which depends on the overall $\ones{b}$-value in the mutant). The potential fitness increase is at most $\loincrease n + n/2$ as the range of $\ones{b}$-values is $n/2$. Likewise, we have $X_i = -1$ if the 0-bit is not flipped, the 1-bit is flipped, and the mutant is accepted (which again depends on the overall $\ones{b}$-value in the mutant). The fitness increase is at least $\loincrease n - n/2$. With the remaining probability we have $X_i = 0$. Hence for global mutations (for local mutations simply drop the $1-1/n$ term)
the total progress in a special step increasing $\lo{a}$ by~$\loincrease$ is stochastically dominated by a sum of independent variables $Y_1, \dots, Y_{n/4}$ where $\Prob{Y_i = \pm 1} = 1/n \cdot (1-1/n) \cdot \pfix(\loincrease n \pm n/2)$ and $Y_i = 0$ with the remaining probability.

There is a bias towards increasing the number of ones due to differences in the arguments of $\pfix$: $\E{Y_i} = 1/n \cdot (1-1/n) \cdot (\pfix(\loincrease n + n/2) - \pfix(\loincrease n - n/2))$.
Using the definition of $\pfix$ and preconditions $\beta = n^{-3/2}$, $N \beta = \ln n$, the bracket is bounded as
\begin{align*}
& \pfix(\loincrease n + n/2) - \pfix(\loincrease n - n/2)\\
=\;& \frac{1-e^{-2\loincrease n^{-1/2} - n^{-1/2}}}{1-n^{-2\loincrease n + n}} - \frac{1-e^{-2\loincrease n^{-1/2} + n^{-1/2}}}{1-n^{-2\loincrease n - n}}\\
=\;& (1+o(1)) \left(\left(1-e^{-2\loincrease n^{-1/2} - n^{-1/2}}\right) - \left(1-e^{-2\loincrease n^{-1/2} + n^{-1/2}}\right)\right)\\
=\;& (1+o(1)) \cdot e^{-2\loincrease n^{-1/2}}  \left(e^{n^{-1/2}}-e^{- n^{-1/2}}\right)\\
\le\;& (1+o(1)) \cdot e^{-2\loincrease n^{-1/2}} \left((1+2n^{-1/2})-(1-n^{-1/2})\right)\\
=\;& (1+o(1)) \cdot e^{-2\loincrease n^{-1/2}} \cdot 3n^{-1/2}
\end{align*}
where in the last inequality we have used $1+x \le e^x$ for all~$x$ and $e^x \le 1+2x$ for $0 \le x \le 1$.

Note that the expectation, and hence the bias, is largest for $\loincrease=1$, in which case we get, using $e^{-2\loincrease n^{-1/2}} \le e^{-2n^{-1/2}} \le 1$,
\begin{equation*}
\E{Y_i} \le (1+o(1)) \cdot 1/n \cdot (1-1/n) \cdot 3n^{-1/2} \le 4n^{-3/2}
\end{equation*}
for $n$ large enough.

The total progress in all $m$ special steps is hence stochastically dominated by a sequence of $m \cdot n/4$ random variables $Y_i$ as defined above, with $\loincrease := 1$.
Invoking Lemma~\ref{lem:sum-of-Yi}, stated in the appendix, with $\delta := n^{3/4}$, the total progress in all special steps is at most
$\delta + m \cdot n/4 \cdot \E{Y_i} = \delta + O(n^{1/2}) =  O(n^{3/4})$
with probability $1-e^{-\Omega(n^{1/2})}$.

Hence the net gain in the number of ones in all special steps is at most $n^{3/4} + O(mn/4 \cdot n^{-3/2}) = O(n^{3/4})$ with probability ${1-e^{-\Omega(n^{1/2})}}$.

Together with all regular steps, the progress on the \onemax part is at most $1.14n/9  + O(n^{3/4})$, which for large enough~$n$ is less than the distance $7n/16 - (n/4+n^{3/4})$ to reach a point with $\ones{b} \ge 7n/16$ from initialisation.
This proves the claim.
\end{proof}

Finally, we put the previous lemmas together into our main theorem that establishes that SSWM can optimise \balance in polynomial time.
\begin{theorem}
With probability $1-e^{-\Omega(n^{1/2})}$ SSWM with $\beta = n^{-3/2}$ and $N \beta = \ln n$ optimises \balance in time $O(n/\beta) = O(n^{5/2})$.
\end{theorem}
\begin{proof}
By Chernoff bounds, the probability that for the initial solution $x_0 = a_0 b_0$ we have $n/4 - n^{3/4} \le \ones{b_0} \le n/4 +n^{3/4}$ is $1-e^{-\Omega(n^{1/2})}$. We assume pessimistically that $n/4 \le \ones{b_0} \le n/4 +n^{3/4}$. Then Lemma~\ref{lem:time-on-om-part} is in force, and with probability $1-e^{-\Omega(n^{1/2})}$ within $T$ relevant steps, $T$ as defined in Lemma~\ref{lem:time-on-lo-part}, SSWM does not reach a trap or a search point with fitness~$0$. Lemma~\ref{lem:time-on-lo-part} then implies that with probability $1-e^{-\Omega(n^{1/2})}$ an optimal solution with $n/2$ leading ones is found.

The time bound follows from the fact that $T = O(n/\beta)$ and that, again by Chernoff bounds, we have at least $T$ relevant steps in $3T$ iterations of SSWM, with probability $1-e^{-\Omega(n^{1/2})}$.
\end{proof}

\section{Conclusions}

The field of evolutionary computation
has matured to the point where techniques can be applied to models of natural evolution. Our analyses have demonstrated that runtime analysis of evolutionary algorithms can be used to analyse a simple model of natural evolution, opening new opportunities for interdisciplinary research with population geneticists and biologists.

Our conclusions are highly relevant for biology, and open the door to the analysis of more complex fitness landscapes in this field and to quantifying the efficiency of evolutionary processes in more realistic scenarios of evolution.
One interesting aspect of our results is that they impose conditions on population size ($N$) and strength of selection ($\beta$) which represent fundamental limits to what is possible by natural selection.
We hope that these results may inspire further research on the similarities and differences between natural and artificial evolution.

From a computational perspective, we have shown that SSWM can overcome obstacles such as posed by $\cliff{d}$ and $\balance$ in different ways to the \ea, due to its non-elitistic selection mechanism. We have seen how the probability of accepting a mutant can be tuned to enable hill climbing, where fitness-proportional selection fails, as well as tunnelling through fitness valleys, where elitist selection fails. For \balance we showed that SSWM can take advantage of information about the steepest gradient. The selection rule in SSWM hence seems to be a versatile and useful mechanism. Future work could investigate its usefulness in the context of population-based evolutionary algorithms.

\bigskip
\textbf{Acknowledgments:}
The research leading to these results has received funding from the European Union Seventh Framework Programme (FP7/2007-2013) under grant agreement no 618091 (SAGE). The authors thank the anonymous GECCO reviewers for their many constructive comments.

\bibliographystyle{abbrv}
\bibliography{literature-short}

\begin{thebibliography}{10}

\bibitem{Auger2011}
A.~Auger and B.~Doerr, editors.
\newblock {\em Theory of Randomized Search Heuristics -- Foundations and Recent
  Developments}.
\newblock Number~1 in Series on Theoretical Computer Science. World Scientific,
  2011.

\bibitem{chastain_algorithms_2014}
E.~Chastain, A.~Livnat, C.~Papadimitriou, and U.~Vazirani.
\newblock Algorithms, games, and evolution.
\newblock {\em Proceedings of the National Academy of Sciences},
  111(29):10620--10623, July 2014.

\bibitem{chatterjee_time_2014}
K.~Chatterjee, A.~Pavlogiannis, B.~Adlam, and M.~A. Nowak.
\newblock The time scale of evolutionary innovation.
\newblock {\em {PLoS} Computational Biology}, 10(9), Sept. 2014.

\bibitem{Corus2014}
D.~Corus, D.-C. Dang, A.~V. Eremeev, and P.~K. Lehre.
\newblock Level-based analysis of genetic algorithms and other search
  processes.
\newblock In {\em PPSN~2014}, pages 912--921. Springer, 2014.

\bibitem{Doerr2011chapter}
B.~Doerr.
\newblock Analyzing randomized search heuristics: Tools from probability
  theory.
\newblock In {\em \cite{Auger2011}}, pages 1--20. World Scientific, 2011.

\bibitem{gillespie_molecular_1984}
J.~H. Gillespie.
\newblock Molecular evolution over the mutational landscape.
\newblock {\em Evolution}, 38(5):1116--1129, 1984.

\bibitem{Jagerskupper2007a}
J.~J{\"a}gersk{\"u}pper and T.~Storch.
\newblock When the plus strategy outperforms the comma strategy and when not.
\newblock In {\em Proc.\ of IEEE FOCI 2007}, pages 25--32. IEEE, 2007.

\bibitem{Jansen2013}
T.~Jansen.
\newblock {\em Analyzing Evolutionary Algorithms. The Computer Science
  Perspective}.
\newblock Springer, 2013.

\bibitem{Jansen2011a}
T.~Jansen, P.~S. Oliveto, and C.~Zarges.
\newblock On the analysis of the immune-inspired {B-Cell} algorithm for the
  {Vertex} {Cover} problem.
\newblock In {\em Proc.\ of ICARIS 2011}, pages 117--131. Springer, 2011.

\bibitem{Jansen2007}
T.~Jansen and I.~Wegener.
\newblock A comparison of simulated annealing with a simple evolutionary
  algorithm on pseudo-{Boolean} functions of unitation.
\newblock {\em Theoretical Computer Science}, 386(1-2):73--93, 2007.

\bibitem{Johannsen2010}
D.~Johannsen.
\newblock {\em Random Combinatorial Structures and Randomized Search
  Heuristics}.
\newblock PhD thesis, Universit{\"a}t des Saarlandes, Saarbr{\"u}cken, Germany
  and the Max-Planck-Institut f{\"u}r Informatik, 2010.

\bibitem{kimura_probability_1962}
M.~Kimura.
\newblock On the probability of fixation of mutant genes in a population.
\newblock {\em Genetics}, 47(6):713--719, 1962.

\bibitem{Lehre2012}
P.~K. Lehre and C.~Witt.
\newblock Black-box search by unbiased variation.
\newblock {\em Algorithmica}, 64(4):623--642, 2012.

\bibitem{NeumannWitt2010}
F.~Neumann and C.~Witt.
\newblock {\em Bioinspired Computation in Combinatorial Optimization --
  Algorithms and Their Computational Complexity}.
\newblock Springer, 2010.

\bibitem{Oliveto2014}
P.~S. Oliveto and D.~Sudholt.
\newblock On the runtime analysis of stochastic ageing mechanisms.
\newblock In {\em Proc.\ of GECCO 2014}, pages 113--120. ACM Press, 2014.

\bibitem{Oliveto2011}
P.~S. Oliveto and C.~Witt.
\newblock Simplified drift analysis for proving lower bounds in evolutionary
  computation.
\newblock {\em Algorithmica}, 59(3):369--386, 2011.

\bibitem{Oliveto2013a}
P.~S. Oliveto and C.~Witt.
\newblock On the runtime analysis of the simple genetic algorithm.
\newblock {\em Theoretical Computer Science}, 545:2--19, 2014.

\bibitem{RohlfshagenLehreYao2009}
P.~Rohlfshagen, P.~K. Lehre, and X.~Yao.
\newblock Dynamic evolutionary optimisation: an analysis of frequency and
  magnitude of change.
\newblock In {\em Proc. of GECCO '09}, pages 1713--1720. ACM Press, 2009.

\bibitem{Rowe2013}
J.~E. Rowe and D.~Sudholt.
\newblock The choice of the offspring population size in the (1,$\lambda$)
  evolutionary algorithm.
\newblock {\em Theoretical Computer Science}, 545:20--38, 2014.

\bibitem{Sudholt2012c}
D.~Sudholt.
\newblock A new method for lower bounds on the running time of evolutionary
  algorithms.
\newblock {\em IEEE Transactions on Evolutionary Computation}, 17(3):418--435,
  2013.

\bibitem{valiant_evolvability_2009}
L.~G. Valiant.
\newblock Evolvability.
\newblock {\em J. {ACM}}, 56(1):3:1--3:21, 2009.

\end{thebibliography}

\clearpage

\onecolumn
\appendix

This appendix contains proofs that were omitted from the main part.

\begin{lemma}
\label{lemma:pfix-strictly-increasing}
$\pfix$ is monotonic for all $N\geq 1$ and strictly increasing for $N>1$	
\end{lemma}
\begin{proof}
	If $N=1$, $\pfix(\s)=1$.  In order to show that $p_\text{fix}(\Delta f)$ is monotonically increasing we show that $\frac{d p_\text{fix}(\Delta f)}{d\Delta f}= \frac{2e^{-2\s}}{1-e^{-2N\s}}-N\frac{e^{-2N\s}(1-e^{-2\s})}{(1-e^{-2N\s})^2} >0$ for all $\df$. For $\s>0$ and $N>1$, we have $e^{-2\s}<1$, and $e^{-2\s}>e^{-2N\s}$. For $\s<0$, the inequalities are reversed.
	If $\s>0$:
\begin{align*}
	& \frac{2e^{-2\s}}{1-e^{-2N\s}}-N\frac{e^{-2N\s}(1-e^{-2\s})}{(1-e^{-2N\s})^2} >0 \\
\Leftrightarrow	& e^{-2\s}\left(1-e^{-2N\s}\right) -Ne^{-2N\s}\left(1-e^{-2\s}\right) >0\\
\Leftrightarrow	& \frac{e^{-2\s}}{e^{-2N\s}} > \frac{1-e^{-2N\s}}{1-e^{-2\s}}.
\end{align*}
Since $\frac{e^{-2\s}}{e^{-2N\s}}>1$ and $ \frac{1-e^{-2N\s}}{1-e^{-2\s}} <1$ this proves the claim for $\s>0$. For $\s<0$ all the inequalities are reversed and $\frac{e^{-2\s}}{e^{-2N\s}}<1$ and $ \frac{1-e^{-2N\s}}{1-e^{-2\s}} >1$.
\end{proof}

\begin{proof}[Proof of Lemma~\ref{lem:mutations-decreasing-ones}]
We follow the proof of Lemma~2 in~\cite{Sudholt2012c}. An offspring with $i+k$ 1-bits is created if and only if there is an integer $j \in \N_0$ such that $j$ 1-bits flip and $k+j$ 0-bits flip.
\begin{align*}
& \mut(i, i+k)\\
=\;& \sum_{j=0}^n \binom{i}{j} \binom{n-i}{k+j} \left(\frac{1}{n}\right)^{k+2j} \left(1-\frac{1}{n}\right)^{n-k-2j}\\
=\;& \left(\frac{1}{n}\right)^k \left(1-\frac{1}{n}\right)^{n-k} \cdot \sum_{j=0}^n \binom{i}{j} \binom{n-i}{k+j} \left(\frac{1}{n-1}\right)^{2j}.
\intertext{Using $\binom{n-i}{k+j} = \frac{1}{(k+j)!} \cdot (n-i) \cdot (n-i-1) \cdot \ldots \cdot (n-i-k-j+1) \le \frac{1}{(k+j)!} \cdot (n-i)^k \cdot (n-i-1)^j$, this is at most}
\le\;& \left(\frac{1}{n}\right)^k \left(1-\frac{1}{n}\right)^{n-k} \cdot \sum_{j=0}^n \frac{(n-i)^k}{j!(k+j)!} \cdot \left(\frac{i (n-i-1)}{(n-1)^2}\right)^j.\\
\intertext{It is easy to see that $\frac{i (n-i-1)}{(n-1)^2} \le \frac{1}{4}$ for all $i$, as the maximum is attained for $i = \frac{n}{2} - \frac{1}{2}$. Hence we get an upper bound of}
\le\;& \left(\frac{n-i}{n}\right)^k \left(1-\frac{1}{n}\right)^{n-k} \cdot \sum_{j=0}^n \frac{4^{-j}}{j!(k+j)!}\\
\intertext{Using $(k+j)! \ge k!(j+1)!$ for all $k \in \N$, $j \in \N_0$,}
\le\;& \left(\frac{n-i}{n}\right)^k \left(1-\frac{1}{n}\right)^{n-k} \cdot \frac{1}{k!} \sum_{j=0}^\infty \frac{4^{-j}}{j!(j+1)!}\\
\le\;& \left(\frac{n-i}{n}\right)^k \left(1-\frac{1}{n}\right)^{n-k} \cdot \frac{1.14}{k!}.
\end{align*}

The proof for mutations decreasing the number of ones follows immediately due to the symmetry $\mut(i,i-k)=\mut(n-i,n-i+k)$.
\end{proof}

\begin{proof}[Proof of Lemma \ref{lem:conditional-mut}]
The proof consists of two parts:\\
1) The probability of improving by $j-i=k$ bits is at least twice as large as the probability of improving by $k+1$ bits, i.e. $\mut(i, i+k)\ge 2 \mut(i, i+k+1)$ for any $0\le i < j \le n$. \\
2)  We use 1) to prove that $  \frac{\mut(i, j)}{\sum_{m=j}^n \mut(i, m)}  \ge \dfrac{1}{2}$.\\

\textbf{Part 1)}
The probability to improve by $k$ bits is
\[
 \mut(i, i+k)= \sum_{l=0}^n \binom{i}{l} \binom{n-i}{k+l} \left(\frac{1}{n}\right)^{k+2l} \left(1-\frac{1}{n}\right)^{n-k-2l}\\
\]
while the probability to improve by $k+1$ bits is
\[
 \mut(i, i+k+1)= \sum_{l=0}^n \binom{i}{l} \binom{n-i}{k+l+1} \left(\frac{1}{n}\right)^{k+2l+1} \left(1-\frac{1}{n}\right)^{n-k-2l-1}.
\]

We want to show that the following is true
\begin{align*}
 \mut(i, i+k)&\ge 2 \mut(i, i+k+1) \Leftrightarrow\\
 \sum_{l=0}^n \binom{i}{l} \binom{n-i}{k+l} \left(\frac{1}{n}\right)^{k+2l} \left(1-\frac{1}{n}\right)^{n-k-2l} &\ge
 2 \sum_{l=0}^n \binom{i}{l} \binom{n-i}{k+l+1} \left(\frac{1}{n}\right)^{k+2l+1} \left(1-\frac{1}{n}\right)^{n-k-2l-1} \Leftrightarrow \\
 \sum_{l=0}^n \binom{i}{l} \binom{n-i}{k+l}  \left(n-1\right)^{n-k-2l} &\ge
 2 \sum_{l=0}^n \binom{i}{l} \binom{n-i}{k+l+1}  \left(n-1\right)^{n-k-2l-1} \Leftrightarrow
\end{align*}
\begin{align*}
  \sum_{l=0}^n \frac{i!(n-i)!}{l!(i-l)!} \frac{(n-1)^{n-k-2l}}{(n-i-k-l-1)!(k+l)!}
  \left[ \frac{1}{(n-i-k-l)} -\frac{2}{(n-1)(k+l+1)} \right]   &\ge 0.
\end{align*}

This holds if following holds for any $0\le l\le n$
\begin{align*}
 \left[ \frac{1}{(n-i-k-l)} -\frac{2}{(n-1)(k+l+1)} \right]   &\ge 0\\
    (n-1)(k+l+1) &\ge 2(n-i-k-l).
\end{align*}
Which  is true for any $k\ge 1$ (thus for any $0\le i<j\le n$).\\

\textbf{Part 2)}
Using the above inequality $\mut(i, i+k)\ge 2 \mut(i, i+k+1)$ we can bound  every possible improvement better than $k$ from above by
\begin{align*}
\mut(i, i+k+l)\le \left(\frac{1}{2}\right)^l \mut(i,i+k)
\end{align*}
for any $0 \le l \le n-i-k$.  This can also be written as
\begin{align*}
\mut(i,j+l)\le \left(\frac{1}{2}\right)^l \mut(i,j)
\end{align*}
for any $0 \le l \le n-j$.
This leads to
\begin{align*}
  \frac{\mut(i, j)}{\sum_{m=j}^n \mut(i, m)} &=  \frac{\mut(i, j)}{\sum_{l=0}^{n-j} \mut(i, j+l)} \\
  & \ge  \frac{\mut(i, j)}{\sum_{l=0}^{n-j} \left(\frac{1}{2}\right)^l \mut(i, j) }\\
&= \frac{1}{\sum_{l=0}^{n-j} \left(\frac{1}{2}\right)^l} = \frac{1}{2-\frac{1}{2^{n-j}}}\ge \frac{1}{2}\\
\end{align*}
which proves Lemma \ref{lem:conditional-mut}.

\end{proof}

\begin{lemma}
\label{lem:sum-of-Yi}
Consider independent random variables $Y_1, \dots, Y_t$ where
\[
Y_i = \begin{cases}  1  &\text{with probability~$p$} \\
 0  &\text{with probability~$1-p-r$}\\
-1  & \text{with probability~$r$}
\end{cases}
\]
then for $Y=\sum_{i=1}^tY_i$ we have $\E{Y}=t(p-r)$ and for every $0 \le \delta \le t(p+r)$
\[  P(Y\ge E(Y)+\delta) \le \; e^{-\Omega \left(t(p+r)\right)} + e^{-\Omega \left(\frac{\delta^2}{t(p+r)}\right)}. \]
\end{lemma}
 \begin{proof}
We imagine $Y_i$ to be drawn in a two-step process: in a first draw with probability $1-p-r$ we set $Y_i = 0$. Otherwise, we have $Y_i \neq 0$ and a second random experiment determines whether $Y_i = 1$ or $Y_i = -1$.
Define indicator variables $X_i \in \{0, 1\}$ for the first experiment: $X_i =1$ if $Y_i \neq 0$. Then $X = \sum_{i=1}^t X_i$
gives the number of events where $Y_i\ne 0$. Furthermore, let $Z_j \in \{-1, +1\}$ be the outcome of the $j$-th instance of the second-type experiment (such an experiment only happens when the first draw determined $Y_i \neq 0$), and $Z=\sum_{j=1}^{X}Z_j$ be the sum of these variables. Since $Z$, in comparison to~$Y$, excludes all summands of value~0, we have $Z = Y$ and hence $\E{Z} = \E{Y} = t(p-r)$.

Is easy to see that \mbox{$(X< 2\E{X})\wedge (Z< \E{Z}+\delta \mid X< 2\E{X})\Rightarrow(Y < \E{Y}+\delta)$} therefore 
\begin{align*}
 P(Y\ge \E{Y}+\delta) \le&\; P ( X\ge 2\E{X}) + P(Z\ge \E{Z}+\delta \mid X< 2\E{X})\\
 \intertext{Now we apply a Chernoff bound to $X$ and a Hoeffding bound to $Z$ for $X \le 2\E{X}$ variables:}
 \le&\; e^{-\frac{4}{3}\E{X}} + e^{-\frac{\delta^2}{4\E{X}}}\\
 =&\; e^{-\Omega(\E{X})} + e^{-\Omega \left(\frac{\delta^2}{\E{X}}\right)}\\
 =&\; e^{-\Omega(t(p+r))} + e^{-\Omega \left(\frac{\delta^2}{t(p+r)}\right)}.\qedhere
\end{align*}
\end{proof}

\end{document}